\pdfoutput=1 
\def\year{2021}\relax
\documentclass[letterpaper]{article} 
\usepackage{aaai21}  
\usepackage{times}  
\usepackage{helvet} 
\usepackage{courier}  
\usepackage[hyphens]{url}  
\usepackage{graphicx} 
\usepackage{natbib}  
\usepackage{caption} 
\usepackage{amsmath}
\usepackage{amssymb}
\usepackage{dsfont}
\usepackage{empheq}
\usepackage{booktabs}
\usepackage{bm}
\usepackage{amsthm}
\theoremstyle{definition}
\newtheorem{theorem}{Theorem}

\newtheorem{lemma}{Lemma}

\newtheorem{assumption}{Assumption}
\newtheorem{claim}{Claim}

\usepackage{algorithm}
\usepackage{algpseudocode}

\newcommand{\tabref}[1]{Table~\ref{#1}}
\newcommand{\figref}[1]{Figure~\ref{#1}}
\newcommand{\expref}[1]{(\ref{#1})}

\newcommand{\lemref}[1]{Lemma~\ref{#1}}
\newcommand{\thmref}[1]{Theorem~\ref{#1}}
\newcommand{\clmref}[1]{Claim~\ref{#1}}

\newcommand{\algoref}[1]{Algorithm~\ref{#1}}

\DeclareMathOperator*{\argmax}{argmax}

\newcommand{\eps}{\varepsilon}

\newcommand{\E}[1]{\mathbb{E}\left[{#1}\right]}

\newcommand{\R}[1]{\mathbb{R}^{#1}}

\renewcommand{\arraystretch}{1.4}

\title{Near-Optimal Regret Bounds for Contextual Combinatorial Semi-Bandits with Linear Payoff Functions}
\author {
    Kei Takemura,\textsuperscript{\rm 1}
    Shinji Ito,\textsuperscript{\rm 1}
    Daisuke Hatano,\textsuperscript{\rm 2}
    Hanna Sumita,\textsuperscript{\rm 3}
    Takuro Fukunaga,\textsuperscript{\rm 4,2,5} \\
    Naonori Kakimura,\textsuperscript{\rm 6}
    Ken-ichi Kawarabayashi\textsuperscript{\rm 7} \\
}
\affiliations {
    \textsuperscript{\rm 1} NEC Corporation \\
    \textsuperscript{\rm 2} RIKEN AIP \\
    \textsuperscript{\rm 3} Tokyo Institute of Technology \\
    \textsuperscript{\rm 4} Chuo University \\
    \textsuperscript{\rm 5} JST PRESTO \\
    \textsuperscript{\rm 6} Keio University \\
    \textsuperscript{\rm 7} National Institute of Informatics \\
    \{kei\_takemura, i-shinji\}@nec.com,
    daisuke.hatano@riken.jp,
    sumita@c.titech.ac.jp,
    fukunaga.07s@chuo-u.ac.jp,
    kakimura@math.keio.ac.jp,
    k\_keniti@nii.ac.jp
}

\begin{document}

\maketitle

\begin{abstract}
  The contextual combinatorial semi-bandit problem with linear payoff functions is a decision-making problem
  in which a learner chooses a set of arms with the feature vectors in each round 
  under given constraints
  so as to maximize the sum of rewards of arms.
  Several existing algorithms have regret bounds that are optimal with respect to the number of rounds $T$.
  However, there is a gap of $\tilde{O}(\max(\sqrt{d}, \sqrt{k}))$ between the current best upper and lower bounds,
  where
  $d$ is the dimension of the feature vectors,
  $k$ is the number of the chosen arms in a round,
  and
  $\tilde{O}(\cdot)$ ignores the logarithmic factors.
  The dependence of $k$ and $d$ is of practical importance
  because $k$ may be larger than $T$ in real-world applications such as recommender systems.
  In this paper,
  we fill the gap by improving the upper and lower bounds.
  More precisely, we show that the C${}^2$UCB algorithm 
  proposed by \citet{qin14}
  has the optimal regret bound $\tilde{O}(d\sqrt{kT} + dk)$ for the partition matroid constraints.
  For general constraints,
  we propose an algorithm that modifies the reward estimates of arms in the C${}^2$UCB algorithm
  and demonstrate that it enjoys the optimal regret bound for a more general problem that can take into account other objectives simultaneously.
  We also show that our technique would be applicable to related problems.
  Numerical experiments support our theoretical results and considerations.
\end{abstract}

\section{Introduction}

\begin{table*}[tb]
  \centering
  \begin{tabular}{@{}lll@{}}
    \toprule
    & Upper bound & Lower bound \\ \midrule
    The best known &
    \begingroup
    \renewcommand{\arraystretch}{1}
    \begin{tabular}{@{}l@{}}
      $\tilde{O}(\max(\sqrt{d}, \sqrt{k})\sqrt{dkT})$ \\
      \citep{qin14,takemura19}
    \end{tabular}
    \endgroup
    & $\Omega(\min(\sqrt{dkT}, kT))$ \citep{kveton15} \\
    This work & $\tilde{O}(d\sqrt{kT} + dk)$ & $\Omega(\min(d\sqrt{kT} + dk, kT))$ \\
    \bottomrule
  \end{tabular}
  \caption{
    Regret bounds for CCS problem ($\tilde{O}(\cdot)$ ignores the logarithmic factors).
  }
  \label{tab:summary}
\end{table*}

This paper investigates the contextual combinatorial semi-bandit problem with linear payoff functions, which we call CCS problem \citep{qin14,takemura19,wen15}.
In this problem, a learner iterates the following process $T$ times.
First, the learner observes $d$-dimensional vectors, called \textit{arms}, and a set of feasible combinations of arms,
where the size of each combination is $k$.
Each arm offers a reward defined by a common linear function over the arms,
but the reward is not revealed to the learner at this point.
Next, the learner chooses a feasible combination of arms.
At the end, the learner observes the rewards of the chosen arms.
The objective of the learner is to maximize the sum of rewards.

The CCS problem includes the linear bandit (LB) problem \citep{abbasi11,agrawal13,auer02,chu11,dani08} and the combinatorial semi-bandit (CS) problem\footnote{
  Here, the CS problem denotes the problem of maximizing the sum of rewards \citep{combes15,kveton14,kveton15},
  while \citet{chen16a,chen16b} deal with a more general objective.
}
\citep{chen16a,chen16b,combes15,gai12,kveton15,wang17,wen15} as special cases.
The difference from the LB problem is that, in the CCS problem, the learner chooses multiple arms at once.
Moreover,
while the given arms are fixed over the rounds and orthogonal to each other in the CS problem,
they may be changed in each round and correlated in the CCS problem.

These differences enable the CCS problem to model more realistic situations of applications such as routing networks \citep{kveton14}, shortest paths \citep{gai12,wen15}, and recommender systems \citep{li10,qin14,wang17}.
For example,
when a recommender system is modeled with the LB problem,
it is assumed that once a recommendation result is obtained, the internal predictive model is updated before the next recommendation.
However, in a real recommender system,
it is more common to update the predictive model after multiple recommendations,
e.g., periodic updates \citep{chapelle11}.
Such a situation can be modeled with the CCS problem, where the number of recommendations between the updates is $k$ and the number of the updates is $T$ \citep{takemura19}\footnote{
  Strictly speaking, the LB problem with periodic updates is a little more restrictive than the CCS problem.
  However, most algorithms for the CCS problem, including the ones proposed in this paper, are applicable to the problem.
}.

As in numerous previous studies on bandit algorithms,
we measure the performance of an algorithm by its regret,
which is the difference between the sum of the rewards of the optimal choices and that of the algorithm's choices.
The existing regret bounds are summarized in 
\tabref{tab:summary},
where $\tilde{O}(\cdot)$ means that the logarithmic factors are ignored.
The best known upper bound on the regret 
is achieved by C${}^2$UCB algorithm, which is given by \citet{qin14}.
\citet{takemura19} refined their analysis
to improve the dependence on other parameters in the regret bound.
The best lower bound is given for the CS problem by \citet{kveton15}.
Note that 
any lower bound for the CS problem is also a lower bound for the CCS problem,
as the CCS problem covers the CS problem.

Although these regret upper and lower bounds match with respect to $T$, there is a gap of $\tilde{O}(\max(\sqrt{d}, \sqrt{k}))$ between them.
In the literature on regret analysis,
the degree of dependence on $T$ in the regret bound usually draws much attention.
However, for the CCS problem, the degree of dependence on $k$ is also important
because there are real-world applications of the CCS problem such that $k$ is large.
In recommender systems with periodic updates, for example,
the number of recommendations between the updates could be large.
An alternative example is the sending promotion problem, in which the number of users to send a promotion at once is much larger than the number of times to send the promotion, i.e., $k \gg T$ \citep{takemura19}.

Our contribution is two-fold.
First, we improve dependence on $d$ and $k$ in both the regret upper and lower bounds.
Our upper and lower bounds match up to logarithmic factors.
Second, we clarify a drawback of the UCB-type algorithms for other related problems and
propose general techniques to overcome the drawback.

To improve the upper bound of the CCS problem,
we first revisit the C${}^2$UCB algorithm.
This algorithm optimistically estimates rewards of arms using confidence intervals of estimates
and then chooses a set of arms based on the optimistic estimates.
Existing upper bounds have $k\sqrt{T}$ factor,
which leads to the gap from the lower bound.
In our analysis, however,
we reveal that the linear dependence on $k$ in the regret comes from the arms of large confidence intervals
and obtain $\tilde{O}(d\sqrt{kT} + dk^2)$ regret by handling such arms separately.
For further improvement,
we focus on the case where
the feasible combinations of arms are given by partition matroids.
We show that the algorithm has the optimal regret bound in this case.
Unfortunately, this analysis cannot apply to the general constraints,
and we do not know whether the C${}^2$UCB algorithm achieves the optimal regret upper bound.
Instead, based on these analyses,
we propose another algorithm that estimates the rewards of arms of large confidence intervals more rigorously;
the algorithm divides the given arms into two groups based on their confidence intervals
and underestimates the rewards of the arms with large confidence intervals.
We show that the proposed algorithm enjoys the optimal regret bound for the CCS problem with any feasible combinations of arms,
and is also optimal for a more general problem
that can take into account both the sum of rewards and other objectives.
For example, recommender systems often require diversity of recommended items \citep{qin13,qin14}.

We support our theoretical analysis through numerical experiments.
We first evaluate the performance of the algorithms on 
instances in which constraints are not represented by the partition matroid.
We observe that the proposed algorithm is superior to the C${}^2$UCB algorithm on these instances, which confirms our theoretical analysis that 
the C${}^2$UCB algorithm may not achieve the optimal regret bound while our proposed algorithm does.
We also evaluate the algorithms on instances with partition matroid constraints.
For these instances, we observe that the C${}^2$UCB and our proposed algorithms perform similarly.

Our theoretical and numerical analyses indicate that the sub-optimality of the C${}^2$UCB algorithm arises from the combinatorial structure of the CCS problem,
i.e., choosing a set of arms in each round.
More precisely,
the existence of an arm with a confidence interval that is too large makes the algorithm choose a bad set of arms.
This is an interesting phenomenon that does not occur in the LB problem (the CCS problem when $k = 1$) or the case of partition matroid constraints.
Since the technique we propose for the CCS problem is so general that it is independent of the linearity of the linear payoff functions,
we believe it could be generalized to overcome the same issue for other semi-bandit problems.

\section{Problem Setting}\label{sec:problem_setting}

In this section, we present the formal definition of the CCS problem and the required assumptions.
The CCS problem consists of $T$ rounds.
Let $N$ denote the number of arms,
and each arm is indexed by an integer in $[N] := \{1, 2, \ldots, N\}$.
We denote by $S_t$ a set of combinations of arms we can choose in the $t$-th round.
We assume that each combination is of size $k$.
Thus, $S_t \subseteq \{I \subseteq [N] \mid |I| = k\}$.

The learner progresses through each round as follows.
At the beginning of the $t$-th round,
the learner observes the set of arms with the associated feature vectors $\{x_t(i)\}_{i \in [N]} \subseteq \R{d}$ and the set of combinations of arms $S_t$.
Then, the learner chooses $I_t \in S_t$.
At the end of the round,
the learner obtains the rewards $\{r_t(i)\}_{i \in I_t}$,
where for all $i \in I_t$, $r_t(i) = {\theta^*}^\top x_t(i) + \eta_t(i)$ for some $\theta^* \in \R{d}$ and $\eta_t(i) \in \R{}$ is a random noise with zero mean.

We evaluate the performance of an algorithm by the expected regret $R(T)$,
which is defined as
\begin{align*}
  R(T) = 
  \sum_{t = 1}^{T} \left(
    \sum_{i \in I^*_t} {\theta^*}^\top x_t(i) -
    \sum_{i \in I_t} {\theta^*}^\top x_t(i)
  \right),
\end{align*}
where $I^*_t = \argmax_{I \in S_t} \sum_{i \in I} {\theta^*}^\top x_t(i)$.

As in previous work \citep{qin14,takemura19},
we assume the following:
\begin{assumption}\label{asm:noise}
  $\forall t \in [T]$ and $\forall i \in I_t$,
  the random noise $\eta_t(i)$ is conditionally $R$-sub-Gaussian,
  i.e.,
  \begin{align*}
    \forall \lambda \in \R{},
    \E{\exp(\lambda\eta_t(i)) \mid \mathcal{F}_t}
    \le \exp\left( \lambda^2R^2 / 2 \right),
  \end{align*}
  where
  $\mathcal{F}_t = \sigma\left(
    \{\{x_s(j)\}_{j \in I_s}\}_{s \in [t]},
    \{\{\eta_s(j)\}_{j \in I_s}\}_{s \in [t-1]}
  \right)$.
\end{assumption}

In addition,
we define the following parameters of the CCS problem:
(i) $L > 0$ such that $\forall i \in [N]$ and $\forall t \in [T]$, $\|x_t(i)\|_2 \le L$,
(ii) $S > 0$ such that $\|\theta^*\|_2 \le S$,
and (iii) $B > 0$ such that $\forall i \in [N]$ and $\forall t \in [T]$, $|{\theta^*}^\top x_t(i)| \le B$.
Note that $LS$ is an obvious upper bound of $B$.

\section{Regret Analysis of the C${}^2$UCB Algorithm}

\begin{algorithm}[tb]
  \caption{C${}^2$UCB \citep{qin14}}
  \label{alg:c2ucb}
  \begin{algorithmic}[1]
    \Require $\lambda > 0$ and $\{\alpha_t\}_{t \in [T]}$ s.t. $\alpha_t > 0$ for all $t \in [T]$.
    \State $V_0 \gets \lambda I$ and $b_0 \gets \bm{0}$.
    \For{$t = 1, 2, \dots, T$}
    \State Observe $\{ x_t(i) \}_{i \in [N]}$ and $S_t$. 
    \State $\hat{\theta}_t \gets V_{t-1}^{-1}b_{t-1}$.
    \For{$i \in [N]$}
    \State $\hat{r}_t(i) \gets \hat{\theta}_t^\top x_t(i) + \alpha_t\sqrt{x_t(i)^\top V_{t-1}^{-1}x_t(i)}$.
    \EndFor
    \State Play a set of arms $I_t = \argmax_{I \in S_t} \sum_{i \in I} \hat{r}_t(i)$ and observe rewards $\{r_t(i)\}_{i \in I_t}$.
    \State $V_t \gets V_{t-1} + \sum_{i \in I_t} x_t(i)x_t(i)^\top$ and $b_t \gets b_{t-1} + \sum_{i \in I_t} r_t(i)x_t(i)$.
    \EndFor
  \end{algorithmic}
\end{algorithm}

\subsection{Existing Analyses}
\citet{qin14} proposed the C${}^2$UCB algorithm (\algoref{alg:c2ucb}),
which chooses a set of arms based on optimistically estimated rewards
in a similar way to other UCB-type algorithms \citep{auer02,chen16b,chu11,li10}.

The C${}^2$UCB algorithm works as follows.
At the beginning of each round,
it constructs an estimator of $\theta^*$ using the arms chosen so far and its rewards (line 3).
It then computes an optimistic reward estimator $\hat{r}_t(i)$ for each observed arm $i$ (line 6),
where $\alpha_t\sqrt{x_t(i)^\top V_{t-1}^{-1} x_t(i)}$ represents the size of the confidence interval of the estimated reward of arm $i$.
Then,
it chooses arms $I_t$ obtained by solving the optimization problem based on $\{\hat{r}_t(i)\}_{i \in [N]}$ (line 8).
Finally,
it observes the reward of the chosen arms
and updates the internal parameters $b_t$ and $V_{t}$ (line 9).

\citet{qin14} showed that the algorithm admits a sublinear regret bound with respect to $T$.
\citet{takemura19} refined their analysis to improve the dependence on $R$, $S$, and $L$ as follows.
Here, for $\delta \in (0, 1)$,
we define $\beta_t(\delta) = R\sqrt{d\log\left( \frac{1 + L^2kt / \lambda}{\delta} \right)} + S\sqrt{\lambda}$.
\begin{theorem}[Theorem 4 of \citet{takemura19}\footnote{
    The original versions of \thmref{thm:c2ucb_old} and \lemref{lem:bound_x_old} assume $L = 1$,
    but it is possible to obtain these results by scaling $x_t(i)$ and $\lambda$ without modifying the proof.
}]\label{thm:c2ucb_old}
  If $\alpha_t = \beta_t(\delta)$ and $\lambda = R^2S^{-2}d$,
  the C${}^2$UCB algorithm has the following regret bound with probability $1 - \delta$:
  \begin{empheq}[left={R(T) = \empheqlbrace}]{alignat=2}
    & \tilde{O}\left(Rd\sqrt{kT}\right) && \quad \mathrm{if}\ \lambda \ge L^2 k \nonumber \\
    & \tilde{O}\left(LSk\sqrt{dT}\right) && \quad \mathrm{otherwise}. \nonumber
  \end{empheq}
\end{theorem}

To prove \thmref{thm:c2ucb_old},
it suffices to bound the cumulative estimating error of rewards, i.e., $\sum_{t \in [T]} \sum_{i \in I_t} (\theta^* - \hat{\theta}_t)^\top x_t(i)$.
Let $\|x_t(i)\|_{V_{t-1}^{-1}}$ denote $\sqrt{x_t(i)^\top V_{t-1}^{-1} x_t(i)}$ for all $i \in [N]$ and $t \in [T]$.
To bound the error,
\citet{takemura19} showed that
\begin{align}\label{eq:1}
  \sum_{t \in [T]} \sum_{i \in I_t} (\theta^* - \hat{\theta}_t)^\top x_t(i)
  \le \beta_T(\delta) \sum_{t \in [T]} \sum_{i \in I_t} \|x_t(i)\|_{V_{t-1}^{-1}}.
\end{align}
The right-hand side is then bounded by the following lemma:
\begin{lemma}[Lemma 5 of \citet{takemura19}]\label{lem:bound_x_old}
  Let $\lambda > 0$.
  Let $\{\{x_t(i)\}_{i \in [k]}\}_{t \in [T]}$ be any sequence such that $x_t(i) \in \R{d}$ and $\|x_t(i)\|_2 \le L$ for all $i \in [k]$ and $t \in [T]$.
  Let $V_t = \lambda I + \sum_{t' \in [t]} \sum_{i \in [k]} x_{t'}(i) x_{t'}(i)^\top$ for all $t \in [T]$.
  Then, we have
  \begin{align*}
    \sum_{t \in [T]} \sum_{i \in [k]} \|x_t(i)\|_{V_{t-1}^{-1}} = \tilde{O}\left( L\sqrt{dk^2T / \lambda} \right).
  \end{align*}
\end{lemma}
This bound is tight up to logarithmic factors
because we have $\sum_{t \in [T]} \sum_{i \in [k]} \|x_t(i)\|_{V_{t-1}^{-1}} = Ldk/\sqrt{\lambda}$
when $T = d$ and $x_t(i) = L e_t$ for all $i \in [k]$ and $t \in [T]$,
where for all $l \in [d]$, $e_l \in \R{d}$ is a vector in which the $l$-th element is 1 and the other elements are 0.

\subsection{Improved Regret Bound}
In this subsection,
we improve the regret bound of the C${}^2$UCB algorithm.
A key observation of our analysis is that
\lemref{lem:bound_x_old} is not tight for sufficiently large $T$.
To improve \lemref{lem:bound_x_old},
we divide $\{\{x_t(i)\}_{i \in [k]}\}_{t \in [T]}$ into two groups:
the family of $x_t(i)$ such that $\|x_t(i)\|_{V_{t-1}^{-1}} \le 1 / \sqrt{k}$, and the others.
As shown in \lemref{lem:bound_x_and_large_x} below,
the sum of $\|x_t(i)\|_{V_{t-1}^{-1}}$ in the former group is $\tilde{O}(\sqrt{dkT})$, which is smaller than \lemref{lem:bound_x_old}.
Moreover, the number of arms in the latter group is shown to be $\tilde{O}(dk)$, which means that not so many arms $x_t(i)$ have large $\|x_t(i)\|_{V_{t-1}^{-1}}$.
\begin{lemma}\label{lem:bound_x_and_large_x}
  Let $\lambda > 0$.
  Let $\{\{x_t(i)\}_{i \in [k]}\}_{t \in [T]}$ be any sequence such that $x_t(i) \in \R{d}$ and $\|x_t(i)\|_2 \le L$ for all $i \in [k]$ and $t \in [T]$.
  Let $V_t = \lambda I + \sum_{t' \in [t]} \sum_{i \in [k]} x_{t'}(i) x_{t'}(i)^\top$ for all $t \in [T]$.
  Then, we have
  \begin{align}
    \sum_{t \in [T]} \sum_{i \in [k]} \min\left( \frac{1}{\sqrt{k}}, \|x_t(i)\|_{V_{t-1}^{-1}} \right) = \tilde{O}\left(\sqrt{dkT}\right) \label{exp:bound_x_main}
  \end{align}
  and
  \begin{align}
    \sum_{t \in [T]} \sum_{i \in [k]} \mathds{1}\left(\|x_t(i)\|_{V_{t-1}^{-1}} > 1/\sqrt{k}\right) = \tilde{O}(dk). \label{exp:bound_num_of_large_x}
  \end{align}
\end{lemma}

Based on \lemref{lem:bound_x_and_large_x},
we can bound the right-hand side of \eqref{eq:1} to obtain a better regret upper bound.
The regret bound given by this theorem is optimal when $LS = B$.

\begin{theorem}\label{thm:c2ucb}
  If $\alpha_t = \beta_t(\delta)$ and $\lambda = R^2S^{-2}d$,
  the C${}^2$UCB algorithm has the following regret bound with probability $1 - \delta$:
  \begin{align*}
    R(T) = \tilde{O}\left(
      Rd\sqrt{kT} + \min\left(LS, Bk\right)dk
    \right).
  \end{align*}
\end{theorem}
\begin{proof}[Proof sketch]
  Let $J_t = \{ i \in [N] \mid \|x_t(i)\|_{V_{t-1}^{-1}} > 1/\sqrt{k} \}$ and $J'_t = I_t \cap J_t$.
  We separate chosen arms into two groups\footnote{
    To show the regret bound of the LinUCB algorithm \citep{chu11,li10},
    i.e., the C${}^2$UCB algorithm for the case $k = 1$,
    \citet{lattimore20} take a similar approach in the note of exercise 19.3.
  }: $\{J'_t\}_{t \in [T]}$ and the remaining arms.
  For $\{J'_t\}_{t \in [T]}$,
  replacing \lemref{lem:bound_x_old} with \lemref{lem:bound_x_and_large_x} in the proof of \thmref{thm:c2ucb_old} gives the first term of the regret bound.
  There are two ways to bound the regret caused by the other group.
  In one way, we use the same proof as the former group, which obtains $\tilde{O}(LSdk)$.
  In the other way,
  by \lemref{lem:bound_x_and_large_x},
  we bound the number of rounds in which the arms of this group are chosen.
  Then,
  we have an upper bound of the regret in a round that is $2Bk$.
  Thus, we obtain $\tilde{O}(Bdk^2)$ in this way.
  The second term of the regret bound can be obtained by combining these two ways.
\end{proof}

Next, we show that \thmref{thm:c2ucb} is better than \thmref{thm:c2ucb_old}.
We first consider the case $\lambda \ge L^2k$.
From the definition of $\lambda$, we have $LSk\sqrt{dT} \le Rd\sqrt{kT}$.
Since \thmref{thm:c2ucb_old} implies $\tilde{O}(Rd\sqrt{kT})$ regret,
it suffices to compare $LSk\sqrt{dT}$ with $\min(LS, Bk)dk$.
If $T < d$, the C${}^2$UCB algorithm has an obvious regret upper bound $2BkT$,
which satisfies $\tilde{O}(LSk\sqrt{dT})$ and $\tilde{O}(\min(LS, Bk)dk)$;
otherwise, we have $LSdk \le LSk\sqrt{dT}$.
In the other case,
\thmref{thm:c2ucb_old} implies $\tilde{O}(LSk\sqrt{dT})$ regret and
we have $Rd\sqrt{kT} \le LSk\sqrt{dT}$.
Thus, it also suffices to compare $LSk\sqrt{dT}$ with $\min(LS, Bk)dk$.
By the discussion in the first case, we obtain the desired result.

\subsection{Improved Regret Bound for the CCS Problem with Partition Matroid Constraints}

In this subsection, we show that the C${}^2$UCB algorithm admits an improved regret upper bound for the CCS problem with the partition matroid constraint, that matches the regret lower bound shown in \tabref{tab:summary}.

Now we define the partition matroid constraint.
Let $\{B_t(j)\}_{j \in [M]}$ be a partition of $[N]$ into $M$ subsets.
Let $\{d_t(j)\}_{j \in [M]}$ be a set of $M$ natural numbers.
Then the partition matroid constraint $S_t$ is defined from $\{B_t(j)\}_{j \in [M]}$ and $\{d_t(j)\}_{j \in [M]}$ as
\begin{align}\label{exp:partition_matroid}
  S_t = \left\{
    I \subseteq [N] \mid
    |I \cap B_t(j)| = d_t(j),
    \forall j \in [M]
  \right\}.
\end{align}
Such $S_t$ is known as the set of the bases of a partition matroid.
It is also known that linear optimization problems on a partition matroid constraint can be solved by the greedy algorithm.
The class of $S_t$ is so large that many fundamental classes are included.
Indeed,
the CCS problem with these constraints leads to the CCS problem with the uniform matroid constraints (i.e., the cardinality constraint) when $M = 1$ and $d_t(1) = k$ for all $t \in [T]$,
and the LB problem with periodic updates when $M = k$ and $d_t(j) = 1$ for all $j \in [M]$ and $t \in [T]$.

We show that the C${}^2$UCB algorithm achieves the optimal regret bound for the CCS problem with constraints satisfying \expref{exp:partition_matroid}:
\begin{theorem}\label{thm:c2ucb_sp}
  Assume that $S_t$ is defined by \expref{exp:partition_matroid} for all $t \in [T]$.
  Then, if $\alpha_t = \beta_t(\delta)$ and $\lambda = R^2S^{-2}d$,
  the C${}^2$UCB algorithm has the following regret bound with probability $1 - \delta$:
  \begin{align*}
    R(T) = \tilde{O}\left(
      Rd\sqrt{kT} + Bdk
    \right).
  \end{align*}
\end{theorem}
\begin{proof}[Proof sketch]
  Recall that $I_t$ is the set of arms chosen by the C${}^2$UCB algorithm in the $t$-th round.
  Let $J_t = \{ i \in [N] \mid \|x_t(i)\|_{V_{t-1}^{-1}}^2 > 1/k \}$ and $J'_t = I_t \cap J_t$.
  As in the proof of \thmref{thm:c2ucb},
  we separate chosen arms into two groups: $I_t \setminus J'_t$ and $J'_t$.
  From the definition of $I_t$ and $J'_t$,
  we obtain $I_t \setminus J'_t = \argmax_{I \in S'_t} \sum_{i \in I} \hat{r}_t(i)$ for all $t \in [T]$,
  where
  \begin{align*}
    S'_t = \left\{
      I \subseteq [N] \setminus J'_t \mid
      \forall j \in [M],
      |I \cap B_t(j)| = \right. \\
      \left. d_t(j) - |B_t(j) \cap J'_t|
    \right\}.
  \end{align*}

  Let $J^*_t$ be a subset of $I^*_t$ that consists of the arms in $I^*_t \cap J'_t$,
  and $|B_t(j)\cap J'_t| - |I^*_t \cap J'_t \cap B_t(j)|$ arms chosen arbitrarily from $I^*_t \cap B_t(j)$
  for each $j \in [M]$.
  Then, $I^*_t \setminus J^*_t \in S'_t$
  and $|J^*_t| = |J'_t|$ for all $t \in [T]$.
  Similar to $I_t$,
  we divide $I^*_t$ into $I^*_t \setminus J^*_t$ and $J^*_t$.
  This gives
  \begin{align*}
    R(T) &=
    \sum_{t \in [T]} \left( \sum_{i \in I_t^* \setminus J_t^*} {\theta^*}^\top x_t(i) - \sum_{i \in I_t \setminus J'_t} {\theta^*}^\top x_t(i) \right) \\
    &+ \sum_{t \in [T]} \left( \sum_{i \in J_t^*} {\theta^*}^\top x_t(i) - \sum_{i \in J'_t} {\theta^*}^\top x_t(i) \right).
  \end{align*}
  The former term in the right-hand side of this equation is $\tilde{O}(Rd\sqrt{kT})$
  by the optimality of $I_t \setminus J'_t$ and the discussion in the proof of \thmref{thm:c2ucb}.
  The latter term is $\tilde{O}(Bdk)$ by the definition of $B$ and \lemref{lem:bound_x_and_large_x}.
\end{proof}

Note that for the LB problem with periodic updates,
the C${}^2$UCB algorithm reduces to the LinUCB algorithm \citep{chu11,li10} with periodic updates, and has the optimal regret bound.
Note also that we can show a similar result for related problems if we have a UCB-type algorithm and an upper bound of the number of chosen arms that have large confidence bounds.

\section{Proposed Algorithm}
In this section,
we propose an algorithm for a more general problem than the CCS problem.
We will show the optimal regret bound of the proposed algorithm for the general problem.

First, let us define the general CCS problem.
Let $X_t = \{x_t(i)\}_{i \in [N]}$ and $r_t^* = \{{\theta^*}^\top x_t(i)\}_{i \in [N]}$ for all $t \in [T]$.
In this problem,
the learner aims to maximize the sum of values $\sum_{t \in [T]} f_{r_t^*, X_t}(I_t)$ instead of the sum of rewards,
where $f_{r_t^*, X_t}(I_t)$ measures the quality of the chosen arms.
As in \citet{qin14},
we assume that the learner has access to an $\alpha$-approximation oracle $\mathcal{O}_S(r, X)$,
which provides $I \in S$ such that $f_{r, X}(I) \ge \alpha \max_{I' \in S} f_{r, X}(I')$ for some $\alpha \in (0, 1]$.
Thus, we evaluate the performance of an algorithm by the $\alpha$-regret $R^\alpha(T)$,
which is defined as
\begin{align*}
  R^\alpha(T) = 
  \sum_{t = 1}^{T} \left(
    \alpha f_{r_t^*, X_t}(I^*_t) - f_{r_t^*, X_t}(I_t)
  \right),
\end{align*}
where $I^*_t = \sum_{i \in I} f_{r_t^*, X_t}(I)$.
Note that the regret of the CCS problem is recovered if $\alpha = 1$ and $f_{r, X}(I)$ is the sum of rewards.
We make the following assumptions that are almost identical to those in \citet{qin14}.

\begin{assumption}\label{asm:monotonisity}
  For all $t \in [T]$ and $I \in S_t$,
  if a pair of rewards $r$ and $r'$ satisfies $r(i) \le r'(i)$ for all $i \in [N]$,
  we have $f_{r, X_t}(I) \le f_{r', X_t}(I)$.
\end{assumption}

\begin{assumption}\label{asm:Lipschitz}
  There exists a constant $C > 0$ such that
  for all $t \in [T]$, all $I \in S_t$, and
  any pair of rewards $r$ and $r'$,
  we have $|f_{r, X}(I) - f_{r', X}(I)| \le C \sum_{i \in I} |r(i) - r'(i)|$.
\end{assumption}

The class of functions that satisfies the assumptions includes practically useful functions.
For example,
the sum of rewards with the entropy regularizer \citep{qin13,qin14},
which has been applied to recommender systems in order to take into account both the sum of rewards and the diversity of the chosen arms,
satisfies the assumptions with $C = 1$.

The proposed algorithm is described in \algoref{alg:proposed}.
When $f_{r, X}(I)$ is the sum of rewards,
the difference between the C${}^2$UCB and the proposed algorithms is the definition of $\hat{r}_t(i)$.
We show the effectiveness of this difference.
In the analysis of the C${}^2$UCB algorithm,
the regret can be decomposed as
\begin{align*}
  R(T) &= \sum_{t \in [T]} \left( \sum_{i \in I_t^*} {\theta^*}^\top x_t(i) - \sum_{i \in I_t} \hat{r}_t(i) \right) \\
  &+ \sum_{t \in [T]} \sum_{i \in I_t} \left( \hat{r}_t(i) - {\theta^*}^\top x_t(i) \right),
\end{align*}
and the first term can be bounded by 0 since $I_t$ is an optimal solution to the problem $\max_{I \in S_t} \sum_{i \in I} \hat{r}_t(i)$.
Then, the right-hand side is bounded by 
\begin{align*}
  R(T) &\le \sum_{t \in [T]} \sum_{i \in I_t \setminus J'_t} (\hat{r}_t(i) - {\theta^*}^\top x_t(i)) \\
  &+ \sum_{t \in [T]} \sum_{i \in J'_t} (\hat{r}_t(i) - {\theta^*}^\top x_t(i)),
\end{align*}
where we recall that $J'_t \subseteq I_t$ is the set of arms such that $\|x_t(i)\|_{V_{t-1}^{-1}} > 1 / \sqrt{k}$.
In the proof of \thmref{thm:c2ucb}, 
the first term of the right-hand side is shown to be $\tilde{O}(Rd\sqrt{kT})$, which is optimal, while the second term can be $\tilde{O}(\max(LS, Bk)dk)$.
The reason the second term is so large is that each arm $i \in J'_t$ may have an overly optimistic reward estimate (i.e., $\hat{r}_t(i)$ may be large).
To overcome this issue, we reduce $\hat{r}_t(i)$ when arm $i$ has an overly optimistic reward estimate, keeping that the reduced value is an optimistic estimate required by UCB-type algorithms.
As described in \algoref{alg:proposed},
we adopt the maximum value of the average reward $B$ as $\hat{r}_t(i)$ when $i \in J_t$.

Similar to the above,
we can show that the proposed algorithm (\algoref{alg:proposed}) has the following regret bound:
\begin{theorem}\label{thm:regret_bound}
  If $\alpha_t = \beta_t(\delta)$ and $\lambda = R^2S^{-2}d$,
  the proposed algorithm has the following regret bound with probability $1 - \delta$:
  \begin{align*}
    R^\alpha(T) = \tilde{O}\left( C \left(
      Rd\sqrt{kT} + Bdk
    \right) \right).
  \end{align*}
\end{theorem}

We show that this regret bound is optimal.
We can define an instance of the general problem with any $C > 0$ from any instance of the CCS problem.
Indeed, for any $C > 0$, we can define $f_{r, X}(I) = C \sum_{i \in I} r(i)$.
Thus, the optimal degree of dependence on $C$ in the regret is linear.
For other parameters, we will show the lower bound in the next section.

\begin{algorithm}[tb]
  \caption{Proposed algorithm}
  \label{alg:proposed}
  \begin{algorithmic}[1]
    \Require $\lambda > 0$ and $\{\alpha_t\}_{t \in [T]}$ s.t. $\alpha_t > 0$ for all $t \in [T]$.
    \State $V_0 \gets \lambda I$ and $b_0 \gets \bm{0}$.
    \For{$t = 1, 2, \dots, T$}
      \State Observe $X_t = \{ x_t(i) \}_{i \in [N]}$ and $S_t$, and let $J_t = \{ i \in [N] \mid x_t(i)^\top V_{t-1}^{-1} x_t(i) > 1/k \}$.
      \State $\hat{\theta}_t \gets V_{t-1}^{-1}b_{t-1}$.
      \For{$i \in [N]$}
        \State If $i \in J_t$ then $\hat{r}_t(i) \gets B$; otherwise $\hat{r}_t(i) \gets \hat{\theta}_t^\top x_t(i) + \alpha_t\sqrt{x_t(i)^\top V_{t-1}^{-1}x_t(i)}$.
      \EndFor
      \State Play a set of arms $I_t = \mathcal{O}_{S_t}(\{\hat{r}_t(i)\}_{i \in [N]}, X_t)$ and observe rewards $\{r_t(i)\}_{i \in I_t}$.
      \State $V_t \gets V_{t-1} + \sum_{i \in I_t} x_t(i)x_t(i)^\top$ and $b_t \gets b_{t-1} + \sum_{i \in I_t} r_t(i)x_t(i)$.
      \EndFor
  \end{algorithmic}
\end{algorithm}

\section{Lower Bounds}

In this section,
we show the regret lower bound that matches the regret upper bound shown in Theorems \ref{thm:c2ucb_sp} and \ref{thm:regret_bound} up to logarithmic factors.
To achieve the lower bound, we mix two types of instances,
which provide $\Omega(Rd\sqrt{kT})$ and $\Omega(Bdk)$ regret, respectively.
While the first type of instance represents the difficulty of learning due to the noise added to the rewards,
the second represents the minimum sample size required to learn the $d$-dimensional vector $\theta^*$ in the CCS problem.

We first consider instances that achieve $\Omega(Rd\sqrt{kT})$ and are analogous to the instances for the LB problem.
Since the lower bound of the LB problem is known to be $\Omega(d\sqrt{T})$ with $R = 1$,
the CCS problem in which the number of arms to select is $kT$ would yield $\Omega(Rd\sqrt{kT})$.
In these instances,
the learner chooses $k$ vertices from a $d$-dimensional hyper cube.
Note that the duplication of vertices is allowed.
\begin{theorem}\label{thm:lb_noise}
  Let $\{x_t(i)\}_{i = (s-1)2^d+ 1}^{s 2^d} = \{-1, 1\}^d$ and $S_t = \{ I \subseteq [k2^d] \mid |I| = k \}$
  for any $s \in [k]$ and $t \in [T]$.
  Let $\Theta = \{-R/\sqrt{kT}, R/\sqrt{kT}\}^d$.
  Assume that $\eta_t(i) \sim \mathcal{N}(0, R^2)$ independently.
  Then, for any algorithm, there exists $\theta^* \in \Theta$ such that $R(T) = \Omega(Rd\sqrt{kT})$.
\end{theorem}
\begin{proof}
  We first consider instances that achieve the lower bound of the LB problem.
  Using the discussion of Theorem 24.1 of \citet{lattimore20},
  we obtain the lower bound of $\Omega(Rd\sqrt{T})$ for a certain $\theta \in \Theta$ when $k=1$.
  Note that this lower bound holds even if the algorithm knows in advance the given set of arms of all rounds.

  Then, we observe that the set of algorithms for the above instances with $kT$ rounds includes any algorithm for the CCS problem, which proves the theorem.
\end{proof}

We next introduce the instances of $\Omega(dk)$, based on the fact that no feedback can be received until $k$ arms are selected in the CCS problem.
More specifically,
these instances consist of $\Theta(d)$ independent 2-armed bandit problems with delayed feedback.
In each problem, the learner suffers $\Omega(Bk)$ regret due to the delayed feedback.

\begin{theorem}\label{thm:lb_learn}
  Let $d = 2d'$ and $S_t = \{ I \subseteq [2k] \mid |I| = k \}$.
  Assume that $\eta_t(i) = 0$.
  Define $\min(d', T)$ groups by dividing rounds.
  For each group $j \in [\min(d', T)]$,
  the given arms are defined as
  $x_t(i) = B\sqrt{d} e_{2j-1}$ for $i \le k$ and
  $x_t(i) = B\sqrt{d} e_{2j}$ for $i > k$,
  where $\{e_l\}_{l \in [d]}$ is the normalized standard basis.
  Let $\Theta = \{-1/\sqrt{d}, 1/\sqrt{d}\}^d$.
  Then, for any algorithm, there exists $\theta^* \in \Theta$ such that $R(T) = \Omega(\min(Bdk, BkT))$.
\end{theorem}
\begin{proof}
  As in Appendix A of \citet{auer02b},
  it is sufficient to consider only the deterministic algorithms.
  In the first round of each group,
  any algorithm selects $k/2$ or more from one of the two types of arms.
  Therefore, we can choose $\theta^* \in \Theta$ so that for each group, the majority type of chosen arms is not optimal,
  in which case the algorithm suffers $\Theta(Bk)$ regret.
\end{proof}

Finally,
by combining the two types of instance above,
we have instances achieving the matching regret lower bound:
\begin{theorem}\label{thm:lower_bound}
  Suppose that $kT = \Omega((Rd/B)^2)$ and $d = 2d'$.
  Then, for any given algorithm,
  if we use instances of \thmref{thm:lb_noise} and \thmref{thm:lb_learn} constructed using different $d'$ dimensions in the first and second halves of the round, respectively,
  that instance achieves the following:
  \begin{align*}
    R(T) = \Omega(\min(Rd\sqrt{kT}+Bdk, BkT)).
  \end{align*}
\end{theorem}
\begin{proof}
  From $kT = \Omega((Rd/B)^2)$,
  we have $|{\theta^*}^\top x_t(i)| < B$ for all $i \in [N]$ and $t \in [T]$.
  Hence, we obtain $R(T) = O(BkT)$.
  Alternatively,
  from \thmref{thm:lb_noise} and \thmref{thm:lb_learn},
  we have $\Omega(Rd\sqrt{kT}+\min(Bdk, BkT))$.
\end{proof}

Note that
we can set $R > 0$ and $B > 0$ arbitrarily in the instances of \thmref{thm:lower_bound},
but $L$ and $S$ are automatically determined as $L = O(\max(1, B\sqrt{d}))$ and $S = O(1)$.

\section{Numerical Experiments}

\begin{table*}
  \centering
  \begin{tabular}{@{}ll@{}}
    \toprule
    Algorithm & Parameters \\ \midrule
    $\eps$-greedy & $\eps = 0.05$ and $\lambda = 1$ \\
    C${}^2$UCB (\algoref{alg:c2ucb}) \citep{qin14}
    & $\lambda = d$ and $\forall t, \alpha_t = \sqrt{d}$ \\
    Thompson sampling \citep{takemura19}
    & $\lambda = d$ and $\forall t, v_t = \sqrt{d}$ \\
    CombLinUCB \citep{wen15}
    & $\lambda = 1$, $\sigma = 1$, and $c = \sqrt{d}$ \\
    CombLinTS \cite{wen15}
    & $\lambda = 1$ and $\sigma = 1$ \\
    Proposed (\algoref{alg:proposed})
    & $\lambda = d$ and $\forall t, \alpha_t = \sqrt{d}$ \\
    \bottomrule
  \end{tabular}
  \caption{
    Algorithms in the numerical experiments.
  }
  \label{tab:algorithms}
\end{table*}

\begin{figure*}[tb]
  \centering
  \includegraphics[width=\hsize]{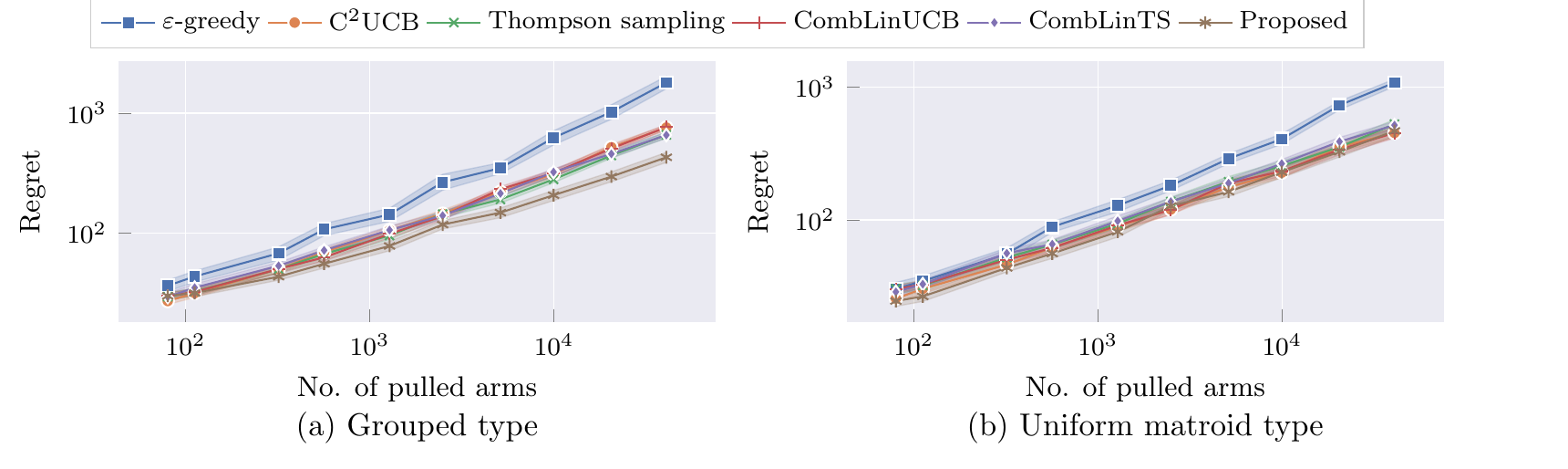}
  \caption{Experimental results.}
  \label{fig:results}
\end{figure*}

\subsection{Setup}
In this section,
we evaluate the performance of
the C${}^2$UCB
and the proposed algorithms through numerical experiments.
Two types of instance are prepared:
one in which the constraints are not represented by the partition matroid
and one in which they are.
We call these types \textit{grouped type} and \textit{uniform matroid type}, respectively.
Our analysis suggests that the C${}^2$UCB algorithm performs well on the uniform matroid type only
and that our proposed algorithm does well on both types.
The aim of our experiments is to verify this.

Let us explain the details of the instances.
The grouped type is given by combining the instances of \thmref{thm:lb_noise} with $d = 4$ and $R = 1$
and an instance defined as follows.
Suppose that $d = 3$, $N = 2k$, and $\theta^* = (0, 0.1, 0.9)^\top$.
Let $f(t)$ be $t - k \lfloor t/k \rfloor$.
The feature vectors are defined as
\begin{empheq}[left={x_t(i) = \empheqlbrace}]{alignat=2}
  & 2^{f(t)} e_{1} && \quad \mathrm{if}\ i = 1 \nonumber \\
  & e_{2} && \quad \mathrm{if}\ 1 < i \le k \nonumber \\
  & e_{3} && \quad \mathrm{if}\ i > k \nonumber
\end{empheq}
for all $t \in [T]$.
The random noise $\eta_t(i)$ follows $\mathcal{N}(0, 1)$ independently for all $t \in [T]$ and $i \in [N]$.
The feasible combinations are defined as
$S_t = \{\{1, 2, \ldots, k\}, \{k+1, k+2, \ldots, 2k\}\}$ for all $t \in [T]$.
Note that this is not represented by the partition matroid.
As for the uniform matroid type,
the feasible combinations are defined as
$S_t = \{ I \subseteq [N] \mid |I| = k \}$ for all $t \in [T]$.
This is one of the uniform matroid constraints, which forms a subclass of partition matroid constraints.
The other parameters are the same as the grouped type.

We start with $k = 2$ and $T = 40$,
and increase $k$ and $T$ so that they satisfy $k = \Theta(T)$.
We run 100 simulations to obtain the means of the regrets.
We evaluate the performance of an algorithm by the means of the regrets for the worst $\theta^*$:
We compare the means for all $\theta^*$ for the largest $kT$
and choose the $\theta^*$ with the largest mean.

We compare the proposed algorithm with five existing algorithms as baselines using the parameters described in \tabref{tab:algorithms}.
The $\eps$-greedy algorithm has two ways of estimating the rewards of given arms:
one is to use the values sampled from $\mathcal{N}(0, 1)$ independently,
and the other is to estimate the rewards as in line 6 of \algoref{alg:c2ucb} with $\alpha_t = 0$.
This algorithm chooses the former way with probability $\eps$ and the latter way otherwise.
Then, it plays a set of arms as in line 8 of \algoref{alg:c2ucb}.

\subsection{Results}
\figref{fig:results}(a) and (b) show the relation between the number of pulled arms (i.e., $kT$) and the regret for the grouped type and the uniform matroid type, respectively.
Error bars represent the standard error.

As we can see in \figref{fig:results}(a),
the regret of the proposed algorithm increased most slowly,
which indicates that
the regrets of the existing and proposed algorithms have different degrees of dependence on the number of pulled arms.
We can explain this phenomenon from the viewpoint of the overly optimistic estimates of rewards.
Since $\|x_t(1)\|_2$ increased exponentially until the $k$-th round,
the C${}^2$UCB algorithm often gave the arm an overly optimistic reward in these rounds.
It follows from this optimistic estimate that
the sum of optimistic rewards in the first group $\{1, 2, \ldots, k\}$ was often greater than that in the other group.
Hence, the C${}^2$UCB algorithm often chose the sub-optimal group and suffered $\Theta(Bk)$ regret in a round.
Note that this phenomenon is almost completely independent of the linearity of the linear payoff function,
which implies that the negative effect of the overly optimistic estimates could appear in UCB-type algorithms for related problems with semi-bandit feedback.

On the other hand, as shown in \figref{fig:results}(b),
the regrets of all the algorithms except the $\eps$-greedy algorithm were almost the same.
This is because
the constraints of the uniform matroid type satisfy the condition \expref{exp:partition_matroid},
and then the C${}^2$UCB algorithm has the optimal regret bound described in \thmref{thm:c2ucb_sp}.
More precisely, as opposed to the grouped type, the regret suffered from the overly optimistic estimates is at most $\Theta(B)$ in a round.

\section{Conclusion}
We have discussed the CCS problem and shown matching upper and lower bounds of the regret.
Our analysis has improved the existing regret bound of the C${}^2$UCB algorithm and clarified the negative effect of the overly optimistic estimates of rewards in bandit problems with semi-bandit feedback.
We have solved this issue in two ways:
introducing partition matroid constraints and providing other optimistic rewards to arms with large confidence intervals.
Our theoretical and numerical analyses have demonstrated the impact of the overly optimistic estimation and the effectiveness of our approaches.

As we discussed,
the negative effect of the overly optimistic estimation could appear in related problems as well.
Since the ideas of our approaches do not depend on the linearity of the linear payoff functions,
we believe they are applicable to overly optimistic estimation in related problems.

Although the proposed algorithm achieves the optimal regret bound,
it uses $B$ explicitly as opposed to the C${}^2$UCB algorithm.
It is an open question whether there exists some algorithm that achieves the optimal regret bound for general constraints without knowledge of the tight upper bound of $B$.

\section*{Acknowledgements}
SI was supported by JST, ACT-I, Grant Number JPMJPR18U5, Japan.
TF was supported by JST, PRESTO, Grant Number JPMJPR1759, Japan.
NK and KK were supported by JSPS, KAKENHI, Grant Number JP18H05291, Japan.

\bibliography{references}

\clearpage
\onecolumn
\appendix

\section{Proofs}
\subsection{Known Results}
Our proofs use the following known results:
\begin{lemma}[Theorem 2 of \citet{abbasi11}]\label{lem:confidence}
  Let $\{F_t\}_{t=0}^\infty$ be a filtration, $\{X_t\}_{t=1}^\infty$ be an $\R{d}$-valued stochastic process such that $X_t$ is $F_{t-1}$-measurable, and $\{\eta_t\}_{t=1}^\infty$ be a real-valued stochastic process such that $\eta_t$ is $F_t$-measurable.
  Let $V = \lambda I$ be a positive definite matrix,
  $V_t = V + \sum_{s \in [t]} X_sX_s^\top$,
  $Y_t = \sum_{s \in [t]} {\theta^*}^\top X_s + \eta_s$ and
  $\hat{\theta}_t = V_{t-1}^{-1}Y_t$.
  Assume for all $t$ that $\eta_t$ is conditionally $R$-sub-Gaussian for some $R > 0$ and
  $\|\theta^*\|_2 \le S$.
  Then, for any $\delta > 0$, with probability at least $1 - \delta$, for any $t \ge 1$,
  \begin{align*}
    \|\hat{\theta}_t - \theta^*\|_{V_{t-1}} \le R \sqrt{2\log\left( \frac{\det(V_{t-1})^{1/2}\det(\lambda I)^{-1/2}}{\delta} \right)} + \sqrt{\lambda} S.
  \end{align*}
  Furthermore, if $\|X_t\|_2 \le L$ for all $t \ge 1$, then with probability at least $1 - \delta$, for all $t \ge 1$,
  \[
    \|\hat{\theta}_t - \theta^*\|_{V_{t-1}} \le R \sqrt{d\log\left( \frac{1 + (t-1)L^2 / \lambda}{\delta} \right)} + \sqrt{\lambda} S.
  \]
\end{lemma}

\begin{lemma}[Lemma 10 of \citet{abbasi11}]\label{lem:det-tr}
  Suppose $X_1$, $X_2$, $\dots$, $X_t \in \R{d}$ and for any $1 \le s \le t$, $\|X_s\|_2 \le L$.
  Let $V_t = \lambda I + \sum_{s \in [t]} X_sX_s^\top$ for some $\lambda > 0$.
  Then,
  \[
    \det(V_{t}) \le (\lambda + tL^2/d)^d.
  \]
\end{lemma}

\subsection{Proof of \lemref{lem:bound_x_and_large_x}}

We first show the lemma below, which proves the first part of \lemref{lem:bound_x_and_large_x}.
Note that this lemma is an extension of Lemma 11 of \citet{abbasi11}.
\begin{lemma}\label{lem:bound_x}
  Let $\{\{x_t(i)\}_{i \in [k]}\}_{t \in [T]}$ be any sequence such that $x_t(i) \in \R{d}$ and $\|x_t(i)\|_2 \le L$ for all $i \in [k]$ and $t \in [T]$.
  Let $V_t = \lambda I + \sum_{s \in [t]} \sum_{i \in [k]} x_s(i) x_s(i)^\top$ with $\lambda > 0$.
  Then, we have
  \begin{align}
    \sum_{t \in [T]} \sum_{i \in [k]} \min\left( \frac{1}{k}, \|x_t(i)\|_{V_{t-1}^{-1}}^2 \right)
    \le 2d \log(1 + L^2kT / (d\lambda)). \label{exp:bound_squared_x}
  \end{align}
  Accordingly, we have
  \begin{align}
    \sum_{t \in [T]} \sum_{i \in [k]} \min\left(\frac{1}{\sqrt{k}}, \|x_t(i)\|_{V_{t-1}^{-1}}\right)
    \le \sqrt{2dkT\log(1 + L^2kT / (d\lambda))}. \label{exp:bound_x}
  \end{align}
\end{lemma}
\begin{proof}
  We have
  \begin{align*}
    \log\det\left( V_T \right)
    &= \log\det\left( V_{T-1} + \sum_{i \in [k]} x_T(i) x_T(i)^\top \right) \\
    &= \log\det\left( V_{T-1} \right) + \log\det\left( I + \sum_{i \in [k]} V_{T-1}^{-1/2}x_T(i) (V_{T-1}^{-1/2}x_T(i))^\top \right) \\
    &= \log\det\left( V_{T-1} \right) + \log\det\left( \sum_{i \in [k]} \frac{1}{k} \left( I + k V_{T-1}^{-1/2}x_T(i) (V_{T-1}^{-1/2}x_T(i))^\top \right) \right) \\
    &\ge \log\det\left( V_{T-1} \right) + \sum_{i \in [k]} \frac{1}{k} \log\det\left( I + k V_{T-1}^{-1/2}x_T(i) (V_{T-1}^{-1/2}x_T(i))^\top \right) \\
    &\ge \log\det\left( \lambda I \right) + \sum_{t \in [T]} \sum_{i \in [k]} \frac{1}{k} \log\det\left( I + k V_{t-1}^{-1/2}x_t(i) (V_{t-1}^{-1/2}x_t(i))^\top \right),
  \end{align*}
  where the first inequality follows from Jensen's inequality applied to the concave function $\log\det(\cdot)$.
  Then, we have
  \begin{align*}
    &\sum_{i \in [k]} \frac{1}{k} \log\det\left( I + k V_{t-1}^{-1/2}x_t(i) (V_{t-1}^{-1/2}x_t(i))^\top \right) \\
    = &\sum_{i \in [k]} \frac{1}{k} \log\left( 1 + k x_t(i)^\top V_{t-1}^{-1}x_t(i) \right) \\
    \ge &\sum_{i \in [k]} \frac{1}{k} \log\left( 1 + \min\left( 1, k x_t(i)^\top V_{t-1}^{-1}x_t(i) \right) \right) \\
    \ge &\sum_{i \in [k]} \frac{1}{2} \min\left( \frac{1}{k}, x_t(i)^\top V_{t-1}^{-1}x_t(i) \right)
  \end{align*}
  for all $t \in [T]$,
  where the second inequality is derived from $2\log(1+x) \ge x$ for all $x \in [0, 1]$.
  Therefore, we obtain
  \begin{align*}
    \sum_{t \in [T]} \sum_{i \in [k]} \frac{1}{2} \min\left( \frac{1}{k}, x_t(i)^\top V_{t-1}^{-1}x_t(i) \right) \le \log\det\left( V_T \right) - \log\det(\lambda I).
  \end{align*}
  Applying \lemref{lem:det-tr} to the above inequality, we obtain \expref{exp:bound_squared_x}.
  Finally, we obtain \expref{exp:bound_x} by applying the Cauchy-Schwarz inequality to \expref{exp:bound_squared_x}.
\end{proof}

Using \lemref{lem:bound_x},
we prove the second part of \lemref{lem:bound_x_and_large_x}.
\begin{lemma}\label{lem:bound_num_of_large_x}
  Let $\{\{x_t(i)\}_{i \in [k]}\}_{t \in [T]}$ be any sequence such that $x_t(i) \in \R{d}$ and $\|x_t(i)\|_2 \le L$ for all $i \in [k]$ and $t \in [T]$.
  Let $V_t = \lambda I + \sum_{s \in [t]} \sum_{i \in [k]} x_s(i) x_s(i)^\top$ with $\lambda > 0$.
  Then, we have
  \begin{align*}
    \sum_{t \in [T]} \sum_{i \in [k]} \mathds{1}\left(\|x_t(i)\|_{V_{t-1}^{-1}} > 1/\sqrt{k}\right) \le 2dk \log(1 + L^2kT / (d\lambda)).
  \end{align*}
\end{lemma}
\begin{proof}
  From \lemref{lem:bound_x},
  we obtain
  \begin{align*}
    \frac{1}{k} \sum_{t \in [T]} \sum_{i \in [k]} \mathds{1}\left(\|x_t(i)\|_{V_{t-1}^{-1}} > 1/\sqrt{k}\right)
    &\le \sum_{t \in [T]} \sum_{i \in [k]} \min\left( \frac{1}{k}, \|x_t(i)\|_{V_{t-1}^{-1}}^2 \right) \\
    &\le 2d \log(1 + L^2kT / (d\lambda)).
  \end{align*}
\end{proof}

\subsection{Proof of \thmref{thm:c2ucb}}

Recall that $\beta_t(\delta) = R\sqrt{d\log\left( \frac{1 + L^2 kt / \lambda}{\delta} \right)} + S\sqrt{\lambda}$ for $\delta > 0$.
\thmref{thm:c2ucb} is a corollary of the following theorem.
\begin{theorem}\label{thm:c2ucb_supp}
  If $\alpha_t = \beta_t(\delta)$,
  the C${}^2$UCB algorithm has the following regret bound with probability $1 - \delta$:
  \begin{align*}
    R(T)
    &= O\left(
      \beta_T(\delta) \sqrt{dkT \log\left(1 + \frac{L^2kT}{d\lambda}\right)} +
      \min\left(\frac{L\beta_T(\delta)}{\sqrt{\lambda}}, Bk\right) dk \log\left(1 + \frac{L^2kT}{d\lambda}\right)
    \right).
  \end{align*}
\end{theorem}

We note that this theorem holds even if the definition of $B$ is replaced with 
\begin{align}
  \forall t \in [T], \forall i, j \in [N],
  |{\theta^*}^\top (x_t(i) - x_t(j))| \le 2B. \label{exp:def_of_B}
\end{align}
Condition \expref{exp:def_of_B} is weaker than the original definition of $B$, as \expref{exp:def_of_B} is derived from $|{\theta^*}^\top x_t(i)| \le B$ for all $t \in [T]$ and $i \in [N]$.

\begin{proof}[Proof of \thmref{thm:c2ucb_supp}]
  Let $J_t = \{ i \in [N] \mid \|x_t(i)\|_{V_{t-1}^{-1}} > 1/\sqrt{k} \}$ and $J'_t = I_t \cap J_t$.
  From \lemref{lem:confidence}, we have $\|\hat{\theta}_t - \theta^*\|_{V_{t-1}} \le \beta_t(\delta)$ for all $t \in [T]$ with probability $1 - \delta$.
  Then, we have
  \begin{align*}
    \hat{r}_t(i) - {\theta^*}^\top x_t(i) 
    &= \beta_t(\delta) \|x_t(i)\|_{V_{t-1}^{-1}} + (\hat{\theta}_t - \theta^*)^\top x_t(i) \\
    &\ge \left( \beta_t(\delta) - \|\hat{\theta}_t - \theta^*\|_{V_t} \right) \|x_t(i)\|_{V_{t-1}^{-1}} \\
    &\ge 0
  \end{align*}
  for all $i \in [N]$ and $t \in [T]$.
  Similarly, we have $\hat{\theta}_t^\top x_t(i) - \beta_t(\delta) \|x_t(i)\|_{V_{t-1}^{-1}} - {\theta^*}^\top x_t(i) \le 0$ for all $i \in [N]$ and $t \in [T]$.
  Thus, it follows that
  \begin{align}
    R(T)
    &= \sum_{t \in [T]} \min\left( \sum_{i \in I_t^*} {\theta^*}^\top x_t(i) - \sum_{i \in I_t} {\theta^*}^\top x_t(i), 2Bk \right) \nonumber \\
    &\le \sum_{t \in [T]} \min\left( \sum_{i \in I_t^*} \hat{r}_t(i) - \sum_{i \in I_t} {\theta^*}^\top x_t(i), 2Bk \right) \nonumber \\
    &\le \sum_{t \in [T]} \min\left( \sum_{i \in I_t} (\hat{r}_t(i) - {\theta^*}^\top x_t(i)), 2Bk \right) \nonumber \\
    &\le \sum_{t \in [T]} \min\left( \sum_{i \in I_t} 2\beta_t(\delta) \|x_t(i)\|_{V_{t-1}^{-1}}, 2Bk \right) \nonumber \\
    &\le \sum_{t \in [T]} \sum_{i \in I_t \setminus J'_t} 2\beta_t(\delta) \|x_t(i)\|_{V_{t-1}^{-1}} +
    \sum_{t \in [T]} \min\left( \sum_{i \in J'_t} 2\beta_t(\delta) \|x_t(i)\|_{V_{t-1}^{-1}}, 2Bk \right), \label{exp:c2ucb_gen_reg_decomp}
  \end{align}
  where the second inequality follows since $I_t = \argmax_{I \in S_t} \sum_{i \in I} \hat{r}_t(i)$,
  and the last inequality is derived from $\beta_t(\delta) \|x_t(i)\|_{V_{t-1}^{-1}} > 0$ for all $\delta \in (0, 1)$, $i \in [N]$, and $t \in [T]$.

  From \lemref{lem:bound_x}, we bound the first term of \expref{exp:c2ucb_gen_reg_decomp} as follows:
  \begin{align*}
    \sum_{t \in [T]} \sum_{i \in I_t \setminus J'_t} 2\beta_t(\delta) \|x_t(i)\|_{V_{t-1}^{-1}}
    &\le 2\beta_T(\delta) \sum_{t \in [T]} \sum_{i \in I_t \setminus J'_t} \|x_t(i)\|_{V_{t-1}^{-1}} \\
    &\le 2\beta_T(\delta) \sqrt{2dkT\log(1 + L^2kT / (d\lambda))}.
  \end{align*}

  We show that the second term of \expref{exp:c2ucb_gen_reg_decomp} is $O\left(\min\left(\frac{L\beta_T(\delta)}{\sqrt{\lambda}}, Bk\right) dk \log\left(1 + \frac{L^2kT}{d\lambda}\right)\right)$
  by bounding this term in two ways.
  We have $\|x_t(i)\|_{V_{t-1}^{-1}}^2 \le L^2 / \lambda$ for all $i \in [N]$ and $t \in [T]$.
  From this fact and \lemref{lem:bound_num_of_large_x}, we obtain
  \begin{align*}
    \sum_{t \in [T]} \min\left( \sum_{i \in J'_t} 2\beta_t(\delta) \|x_t(i)\|_{V_{t-1}^{-1}}, 2Bk \right)
    &\le \sum_{t \in [T]} \sum_{i \in J'_t} 2\beta_t(\delta) \|x_t(i)\|_{V_{t-1}^{-1}} \\
    &\le 2\beta_T(\delta) \sum_{t \in [T]} \sum_{i \in J'_t} \|x_t(i)\|_{V_{t-1}^{-1}} \\
    &\le 4\beta_T(\delta) \frac{Ldk}{\sqrt{\lambda}} \log(1 + L^2kT / (d\lambda)).
  \end{align*}
  Alternatively,
  it follows from \lemref{lem:bound_num_of_large_x} that
  \begin{align*}
    \sum_{t \in [T]} \min\left( \sum_{i \in J'_t} 2\beta_t(\delta) \|x_t(i)\|_{V_{t-1}^{-1}}, 2Bk \right)
    &\le \sum_{t \in [T]} 2|J'_t| Bk \\
    &\le 4B dk^2 \log(1 + L^2kT/(d\lambda)).
  \end{align*}
  Combining these inequalites, we obtain \thmref{thm:c2ucb_supp}.
\end{proof}

\subsection{Proof of \thmref{thm:c2ucb_sp}}

\thmref{thm:c2ucb_sp} is a corollary of the theorem below.
Note that this theorem holds when $B$ is a parameter satisfying \expref{exp:def_of_B}.
\begin{theorem}
  Assume that $S_t$ satisfies \expref{exp:partition_matroid} for a partition $\{B_t(j)\}_{j \in [M]}$ and $\{d_t(j)\}_{j \in [M]}$ for all $t \in [T]$.
  Then, if $\alpha_t = \beta_t(\delta)$,
  the C${}^2$UCB algorithm has the following regret bound with probability $1 - \delta$:
  \begin{align*}
    R(T) = O\left(
      \beta_T(\delta) \sqrt{dkT \log\left(1 + \frac{L^2kT}{d\lambda}\right)} +
      Bdk \log\left(1 + \frac{L^2kT}{d\lambda}\right)
    \right).
  \end{align*}
\end{theorem}

\begin{proof}
  Let $J_t = \{ i \in [N] \mid \|x_t(i)\|_{V_{t-1}^{-1}} > 1/\sqrt{k} \}$ and $J'_t = I_t \cap J_t$.
  We separate chosen arms into two groups:
  $\{J'_t\}_{t \in [T]}$ and the other arms.
  From the optimality of $I_t$,
  $I_t \cap B_t(j)$ is the top $d_t(j)$ arms in terms of $\hat{r}_t(\cdot)$ in $B_t(j)$ for each $j \in [M]$.
  Thus, we obtain
  \begin{align}
    I_t \setminus J'_t = \argmax_{I \in S'_t} \sum_{i \in I} \hat{r}_t(i) \label{exp:opt_r_hat}
  \end{align}
  for all $t \in [T]$,
  where
  \begin{align*}
    S'_t = \left\{ 
      I \subseteq [N] \setminus J'_t \mid
      |I \cap B_t(j)| = d_t(j) - |B_t(j) \cap J'_t|,
      \forall j \in [M]
    \right\}.
  \end{align*}

  Let $J^*_t$ be a subset of $I^*_t$ that consists of the arms in $I^*_t \cap J'_t$,
  and $|B_t(j)\cap J'_t| - |I^*_t \cap J'_t \cap B_t(j)|$ arms chosen arbitrarily from $I^*_t \cap B_t(j)$
  for each $j \in [M]$.
  Then, $I^*_t \setminus J^*_t \in S'_t$
  and $|J^*_t| = |J'_t|$ for all $t \in [T]$.
  Similarly to $I_t$,
  we divide $I^*_t$ into $I^*_t \setminus J^*_t$ and $J^*_t$.
  This gives
  \begin{align}
    R(T) =
    \sum_{t \in [T]} \left( \sum_{i \in I_t^* \setminus J_t^*} {\theta^*}^\top x_t(i) - \sum_{i \in I_t \setminus J'_t} {\theta^*}^\top x_t(i) \right) +
    \sum_{t \in [T]} \left( \sum_{i \in J_t^*} {\theta^*}^\top x_t(i) - \sum_{i \in J'_t} {\theta^*}^\top x_t(i) \right). \label{exp:c2ucb_sp_reg_decomp}
  \end{align}

  For the second term of \expref{exp:c2ucb_sp_reg_decomp},
  from \lemref{lem:bound_num_of_large_x} and the definition of $B$,
  we have
  \begin{align*}
    \sum_{t \in [T]} \left( \sum_{i \in J_t^*} {\theta^*}^\top x_t(i) - \sum_{i \in J'_t} {\theta^*}^\top x_t(i) \right)
    \le 4Bdk \log(1 + L^2kT / (d\lambda)).
  \end{align*}

  It remains to be shown that the first term of \expref{exp:c2ucb_sp_reg_decomp} is $O\left( \beta_T(\delta) \sqrt{dkT \log\left(1+ \frac{L^2kT}{d\lambda}\right)} \right)$.
  By the same discussion as the proof of \thmref{thm:c2ucb_supp},
  we obtain $\hat{r}_t(i) \ge {\theta^*}^\top x_t(i)$ for all $i \in [N]$ and $t \in [T]$ with probability $1 - \delta$.
  Then, we have
  \begin{align*}
    &\sum_{t \in [T]} \left( \sum_{i \in I_t^* \setminus J_t^*} {\theta^*}^\top x_t(i) - \sum_{i \in I_t \setminus J'_t} {\theta^*}^\top x_t(i) \right) \\
    \le &\sum_{t \in [T]} \left( \sum_{i \in I_t^* \setminus J_t^*} \hat{r}_t(i) - \sum_{i \in I_t \setminus J'_t} {\theta^*}^\top x_t(i) \right) \\
    \le &\sum_{t \in [T]} \sum_{i \in I_t \setminus J'_t} \left( \hat{r}_t(i) - {\theta^*}^\top x_t(i) \right) \\
    = &\sum_{t \in [T]} \sum_{i \in I_t  \setminus J'_t} \left( \hat{r}_t(i) - \hat{\theta}_t^\top x_t(i) \right) + \sum_{t \in [T]} \sum_{i \in I_t  \setminus J'_t} \left( \hat{\theta}_t - \theta^* \right)^\top x_t(i),
  \end{align*}
  where the second inequality is derived from \expref{exp:opt_r_hat}.
  We define $R^{alg}(T)$ and $R^{est}(T)$ as
  \begin{align*}
    R^{alg}(T) &= \sum_{t \in [T]} \sum_{i \in I_t  \setminus J'_t} \left( \hat{r}_t(i) - \hat{\theta}_t^\top x_t(i) \right) \quad \mathrm{and}\\
    R^{est}(T) &= \sum_{t \in [T]} \sum_{i \in I_t  \setminus J'_t} \left( \hat{\theta}_t - \theta^* \right)^\top x_t(i).
  \end{align*}
  From \clmref{clm:r^alg} and \clmref{clm:r^est} below,
  we can bound $R^{alg}(T)$ and $R^{est}(T)$, respectively, which gives the desired bound of the first term of \expref{exp:c2ucb_sp_reg_decomp}.
\end{proof}

\begin{claim}\label{clm:r^alg}
  \begin{align*}
    R^{alg}(T) = O\left( \beta_T(\delta) \sqrt{dkT \log\left(1+ \frac{L^2kT}{d\lambda}\right)} \right)
  \end{align*}
\end{claim}
\begin{proof}[Proof of \clmref{clm:r^alg}]
  Recall that $J_t = \{ i \in [N] \mid \|x_t(i)\|_{V_{t-1}^{-1}} > 1/\sqrt{k} \}$ and $J'_t = I_t \cap J_t$.
  We can bound $R^{alg}(T)$ as follows:
  \begin{align*}
    R^{alg}(T)
    &= \sum_{t \in [T]} \sum_{i \in I_t \setminus J'_t} \left( \hat{r}_t(i) - \hat{\theta}_t^\top x_t(i) \right) \\
    &= \sum_{t \in [T]} \sum_{i \in I_t \setminus J'_t} \beta_t(\delta) \|x_t(i)\|_{V_{t-1}^{-1}} \\
    &\le \beta_T(\delta) \sum_{t \in [T]} \sum_{i \in I_t \setminus J'_t} \|x_t(i)\|_{V_{t-1}^{-1}}.
  \end{align*}
  Then, \lemref{lem:bound_x} gives the desired result.
\end{proof}

\begin{claim}\label{clm:r^est}
  \begin{align*}
    R^{est}(T) = O\left( \beta_T(\delta) \sqrt{dkT \log\left(1+ \frac{L^2kT}{d\lambda}\right)} \right)
  \end{align*}
\end{claim}
\begin{proof}[Proof of \clmref{clm:r^est}]
  Recall that $J_t = \{ i \in [N] \mid \|x_t(i)\|_{V_{t-1}^{-1}} > 1/\sqrt{k} \}$ and $J'_t = I_t \cap J_t$.
  It follows from \lemref{lem:confidence} that
  with probability $1 - \delta$,
  \begin{align*}
    R^{est}(T) &= \sum_{t \in [T]} \sum_{i \in I_t \setminus J'_t} (\hat{\theta}_t - \theta^*)^\top x_t(i) \\
    &\le \sum_{t \in [T]} \sum_{i \in I_t \setminus J'_t} \|\hat{\theta}_t - \theta^*\|_{V_{t-1}} \|x_t(i)\|_{V_{t-1}^{-1}} \\
    &\le \sum_{t \in [T]} \sum_{i \in I_t \setminus J'_t} \beta_t(\delta) \|x_t(i)\|_{V_{t-1}^{-1}} \\
    &\le \beta_T(\delta) \sum_{t \in [T]} \sum_{i \in I_t \setminus J'_t} \|x_t(i)\|_{V_{t-1}^{-1}}.
  \end{align*}
  Applying \lemref{lem:bound_x} to the above, we obtain the desired result.
\end{proof}

\subsection{Proof of \thmref{thm:regret_bound}}

\thmref{thm:regret_bound} is a corollary of the theorem below.
In this subsection, we prove this theorem.
\begin{theorem}
  If $\alpha_t = \beta_t(\delta)$,
  the proposed algorithm has the following regret bound with probability $1 - \delta$:
  \begin{align*}
    R^\alpha(T)
    &= O\left( C \left(
      \beta_T(\delta) \sqrt{dkT \log\left(1 + \frac{L^2kT}{d\lambda}\right)} +
      B dk \log\left(1 + \frac{L^2kT}{d\lambda}\right)
      \right) \right).
  \end{align*}
\end{theorem}

\begin{proof}
  Let $J_t = \{ i \in [N] \mid \|x_t(i)\|_{V_{t-1}^{-1}} > 1/\sqrt{k} \}$ and $J'_t = I_t \cap J_t$.
  Similarly to the discussion in the proof of \thmref{thm:c2ucb_supp},
  we have ${\theta^*}^\top x_t(i) \le \hat{r}_t(i)$ for all $i \in [N] \setminus J_t$ and $t \in [T]$ with probability $1 - \delta$.
  Furthermore, from the definition of $B$,
  we also have ${\theta^*}^\top x_t(i) \le \hat{r}_t(i)$ for all $i \in J_t$ and $t \in [T]$.
  Thus, we obtain $\hat{r}_t(i) \ge {\theta^*}^\top x_t(i)$ for all $i \in [N]$ and $t \in [T]$.

  Let $\hat{r}_t = \{\hat{r}_t(i)\}_{i \in [N]}$ for all $t \in [T]$.
  Recall that $r_t^* = \{{\theta^*}^\top x_t(i)\}_{i \in [N]}$.
  Then, we have
  \begin{align*}
    R^\alpha(T)
    &= \sum_{t \in [T]} (\alpha f_{r_t^*, X_t}(I^*_t) - f_{r_t^*, X_t}(I_t)) \nonumber \\
    &\le \sum_{t \in [T]} (\alpha f_{\hat{r}_t, X_t}(I^*_t) - f_{r_t^*, X_t}(I_t)) \nonumber \\
    &\le \sum_{t \in [T]} (f_{\hat{r}_t, X_t}(I_t) - f_{r_t^*, X_t}(I_t)) \nonumber \\
    &\le C \sum_{t \in [T]} \sum_{i \in I_t} (\hat{r}_t(i) - {\theta^*}^\top x_t(i)),
  \end{align*}
  where
  the second inequality is derived from the definition of $I_t$
  and
  the first and third inequalities are derived from Assumptions \ref{asm:monotonisity} and \ref{asm:Lipschitz} and the optimisticity of $\hat{r}_t$ for all $t \in [T]$.
  From the above discussion,
  we obtain
  \begin{align*}
    R^\alpha(T) \le C \left( R^{alg}(T) + R^{est}(T) + R^{con}(T) \right),
  \end{align*}
  where
  \begin{align*}
    R^{alg}(T) &= \sum_{t \in [T]} \sum_{i \in I_t  \setminus J'_t} \left( \hat{r}_t(i) - \hat{\theta}_t^\top x_t(i) \right), \\
    R^{est}(T) &= \sum_{t \in [T]} \sum_{i \in I_t  \setminus J'_t} \left( \hat{\theta}_t - \theta^* \right)^\top x_t(i), \quad \mathrm{and} \\
    R^{con}(T) &= \sum_{t \in [T]} \sum_{i \in J'_t} \left( \hat{r}_t(i) - {\theta^*}^\top x_t(i) \right).
  \end{align*}
  We will bound each of them separately.

  For $R^{con}(T)$,
  it follows from the definition of $B$ and \lemref{lem:bound_num_of_large_x} that
  \begin{align*}
    R^{con}(T)
    &\le 2B \sum_{t \in [T]} |J'_t| \\
    &\le 4B dk\log(1 + L^2kT/(d\lambda)).
  \end{align*}
  Similarly to \clmref{clm:r^alg} and \clmref{clm:r^est}, we obtain
  \begin{align*}
    R^{alg}(T) &= O\left( \beta_T(\delta) \sqrt{dkT \log\left(1+ \frac{L^2kT}{d\lambda}\right)} \right) \quad \mathrm{and} \\
    R^{est}(T) &= O\left( \beta_T(\delta) \sqrt{dkT \log\left(1+ \frac{L^2kT}{d\lambda}\right)} \right).
  \end{align*}
  This completes the proof.
\end{proof}

\end{document}